\documentclass[final,1p,times]{elsarticle}

\usepackage{amsmath,amssymb,amsthm,mathrsfs}
\usepackage[colorlinks=false,bookmarks=true,pdfborder={0 0 0},
    breaklinks=true,bookmarksnumbered=true]{hyperref}

\graphicspath{{figures/}} 

\biboptions{sort&compress}

\newcommand{\naturals}{\ensuremath{\mathbb{N}}}	      
\newcommand{\real}{\ensuremath{\mathbb{R}}}
\newcommand{\abs}[1]{\lvert#1\rvert}
\newcommand{\norm}[1]{\lVert#1\rVert}

\newcommand{\rank}{\mathrm{rank} \; }

\newcommand{\vol}{V}
\newcommand{\dvol}{dV}		% Riemannian volume form (was dV_g)
\newcommand{\diam}{d}  % (metric) diameter
\newcommand{\inj}{\mathrm{inj}}    % injectivity radius
\newcommand{\can}{\mathrm{can}}  % canonical metric
\newcommand{\ricci}{\mathrm{Ric}}

\theoremstyle{definition} 	
\theoremstyle{plain} 		\newtheorem{lemma}{Lemma}
\theoremstyle{definition} 	\newtheorem{remark}{Remark}
\theoremstyle{plain} 		\newtheorem{theorem}{Theorem}
\theoremstyle{plain} 		\newtheorem{proposition}{Proposition}
\theoremstyle{plain} 		\newtheorem{corollary}{Corollary}

\journal{Applied and Computational Harmonic Analysis}

\begin{document}

\begin{frontmatter}



\title{The embedding dimension of Laplacian eigenfunction maps}

% \author{Jonathan Bates\corref{cor1}\fnref{label2}}
\author{Jonathan Bates \fnref{label2}}
\ead{jonrbates{@}gmail.com}
\fntext[label2]{Now a Postdoctoral Fellow in Medical Informatics 
at VA Connecticut, West Haven, CT 06516, USA}
\address{Department of Mathematics, 
Florida State University,
Tallahassee, FL 32306, USA}

\begin{abstract}
Any closed, connected Riemannian manifold $M$ can be smoothly embedded 
by its Laplacian eigenfunction maps into $\ensuremath{\mathbb{R}}^m$ for some $m$.
We call the smallest such $m$ the maximal embedding dimension of $M$.
We show that the maximal embedding dimension of $M$ is bounded from above
by a constant depending only on the dimension of $M$, a lower bound for injectivity radius, 
a lower bound for Ricci curvature, and a volume bound.
We interpret this result for the case of surfaces isometrically immersed in $\ensuremath{\mathbb{R}}^3$, 
showing that the maximal embedding dimension only depends on bounds 
for the Gaussian curvature, mean curvature, and surface area.
Furthermore, we consider the relevance of these results for shape registration.
\end{abstract}


\end{frontmatter}

% \linenumbers

%----------------------------------------------------------------------------------------------------------------------------
% MAIN TEXT
%----------------------------------------------------------------------------------------------------------------------------

\section{Introduction}

Let $M = (M,g)$ be a closed (compact, without boundary), connected Riemannian manifold;
we assume both $M$ and $g$ are smooth.
The Laplacian of $M$ is a differential operator given by
$\Delta := -\mathrm{div} \circ \mathrm{grad}$, where $\mathrm{div}$ 
and $\mathrm{grad}$ are the Riemannian divergence and gradient, respectively.
Since $M$ is compact and connected, $\Delta$ has a discrete spectrum 
$\{\lambda_j\}_{j \in \naturals}$,
$0 = \lambda_0 < \lambda_1 \leq \lambda_2 \leq \dotsb \uparrow \infty$.
We may choose an orthonormal basis for $L^2(M)$ of eigenfunctions
$\{ \varphi_j \}_{j \in \naturals}$ of $\Delta$, where $\Delta \varphi_j = \lambda_j \varphi_j$,
$\varphi_j \in C^{\infty}(M)$, $\varphi_0 \equiv \vol(M)^{-1/2}$. 
Here, $\vol(M)$ denotes the volume of $M$ with respect to the 
canonical Riemannian measure $\vol = \vol_{(M,g)}$.

We consider maps of the form
\begin{equation}
\label{E:emap}
\begin{split}
\Phi^m : M &\longrightarrow \real^m \\
x &\longmapsto \{ \, \varphi_j(x) \, \}_{1 \leq j \leq m} \,.
\end{split}
\end{equation}
If $\Phi^m : M \to \real^m$ happens to be a smooth embedding,
then we call it an \emph{$m$-dimensional eigenfunction embedding} of $M$.  
The smallest number $m$ for which $\Phi^m$
is an embedding for \emph{some} choice of basis $\{ \varphi_j \}_{j \in \naturals}$
will herein be called the \emph{embedding dimension} of $M$, and
the smallest number $m$ for which $\Phi^m$
is an embedding for \emph{every} choice of basis $\{ \varphi_j \}_{j \in \naturals}$
will be called the \emph{maximal embedding dimension} of $M$. 
Our aim is to establish a (qualitative) bound for the maximal embedding dimension 
of a given Riemannian manifold in terms of basic geometric data.

That finite eigenfunction maps of the form \eqref{E:emap} yield smooth embeddings 
for large enough $m$ appears in a few papers in the spectral geometry literature.
Abdallah \cite{abdallah12} traces this fact back to B\'erard \cite{berardvolume}.
To our knowledge, the latest embedding result is given in Theorem 1.3 in
Abdallah \cite{abdallah12}, who shows that 
when $(M,g(t))$ is a family of Riemannian manifolds with $g(t)$ 
analytic in a neighborhood of $t =0$, then there are
$\epsilon > 0$, $m \in \naturals$, and eigenfunctions 
$\{ \varphi_j(t) \}_{1 \leq j \leq m}$ of $\Delta_{g(t)}$ such that
\begin{equation}
\begin{split}
(M,g(t)) &\longrightarrow \real^m \\
x &\longmapsto \{ \, \varphi_j(x; t) \, \}_{1 \leq j \leq m}
\end{split}
\end{equation}
is an embedding for all $t \in (-\epsilon,\epsilon)$.
The proof does not suggest how topology and geometry 
determine the embedding dimension, however.

Jones, Maggioni, and Schul \cite{jms08,jms10} have studied local 
properties of eigenfunction maps, and their results are essential to the proof of our main result.
In particular, they show that at $z \in M$, for an appropriate choice of weights 
$a_1,\dotsc,a_n \in \real$ and eigenfunctions $\varphi_{j_1},\dotsc,\varphi_{j_n}$, 
one has a coordinate chart $(U,\Phi_{a})$ around $z \in M$, where
$\Phi_{a}(x) := (a_1 \varphi_{j_1}(x), \dotsc, a_n \varphi_{j_n}(x))$, 
satisfying $\norm{\,\Phi_{a}(x)-\Phi_{a}(y)\,}_{\real^n} \sim d_M(x,y)$ 
for all $x,y \in U$. A more explicit statement of this result is given below.

Minor variants of such eigenfunction maps have been used in a variety of contexts.
For example, spectral embeddings
\begin{equation}
\label{E:specemb}
\begin{split}
M &\longrightarrow \ell^2 \\
x &\longmapsto \{ \,e^{-\lambda_j t/2} \varphi_j(x)\, \}_{j \in \naturals} 
\quad (t >0)
\end{split}
\end{equation}
have been used to embed closed Riemannian manifolds into 
the Hilbert space $\ell^2$ (i.e.\ square summable sequences with the usual inner product)
in B\'erard, Besson, and Gallot \cite{bbg88,bbg94};
Fukaya \cite{fukaya87}; Kasue and Kumura, e.g.\ \cite{kasuekumura94,kasuekumura96};
Kasue, Kumura, and Ogura \cite{kasue97}; 
Kasue, e.g.\ \cite{kasue02,kasue06}; and Abdallah \cite{abdallah12}.

Relatives of the eigenfunction maps, or a discrete counterpart, 
have been studied for data parametrization 
and dimensionality reduction, e.g.\ \cite{belkin-eigenmaps01,bai-hancock04,
lafon-thesis,coifman-lafon06,levy06,rustamov07};
for shape distances, e.g.\ \cite{jain-retrieval07,elghawalby-hancock08,
bates-icpr10,memoli-spectral}; and
for shape registration, e.g.\ \cite{carcassoni-hancock00,jain-correspondence07,
mateus-etal08,liu-icpr08,bates-isbi09,reuter-ijcv10,sharma-horaud}.
In particular, in the data analysis community,
\eqref{E:emap} is known as the \emph{eigenmap} \cite{belkin-eigenmaps01},
\eqref{E:specemb} is known as the \emph{diffusion map} \cite{lafon-thesis,coifman-lafon06},
and  $x \mapsto \{ \lambda_j^{-1/2} \varphi_j(x) \}$
is known as the \emph{global point signature} \cite{rustamov07}.
These maps are all equivalent up to an invertible linear transformation.
Hence, any embedding result applies to all of them.
For an overview of spectral geometry in shape and data analysis,
we refer the reader to M\'emoli \cite{memoli-spectral}.

There seem to be no rules for choosing the number of eigenfunctions
to use for a given application.  While not all applications require an (injective) embedding
of data, many eigenfunction-based shape registration methods do,
e.g.\ \cite{jain-correspondence07,mateus-etal08,liu-icpr08,bates-isbi09,
reuter-ijcv10,sharma-horaud}, as we explain in Section \ref{S:implications} below.
In the discrete setting one can write an algorithm to determine the smallest $m$ 
for which $\Phi^m : M \to \real^m$ is an embedding, although such an approach
may become computationally intensive.  For example, if $M$ is represented 
as a polyhedral surface, one may write an algorithm to check for self-intersections
of polygon faces in the image $\Phi^m : M \to \real^m$.
The fail-proof approach is to use all eigenfunctions, in which case
one is assured an embedding.
This approach is mentioned for point cloud data in Coifman and Lafon \cite{coifman-lafon06}.
Specifically, they bound the maximal embedding dimension from above by
the size of the full point sample.  This becomes computationally demanding,
however, especially in applications where one must solve an optimization problem 
over all eigenspaces, 
e.g.\ \cite{jain-correspondence07,mateus-etal08,bates-icpr10,reuter-ijcv10},
as we discuss in Section \ref{S:implications}.
Under the assumption that the shape or data is a sample drawn from
some Riemannian manifold, we expect the embedding dimension of the sample 
to depend only on the topology and geometry of the manifold and the quality of the sample 
(e.g.\ covering radius).
In this note we consider what topological and geometric data influence the embedding dimension
of the underlying manifold.

The 3D image $\Phi^3 : M \to \real^3$ of a hippocampus 
is plotted in Figure \ref{figure:hippocampus_registration}.
It is not clear from inspection whether the 3D image has self-intersections.
To use the $N$-D image for registration as in \cite{carcassoni-hancock00,jain-correspondence07,
mateus-etal08,liu-icpr08,bates-isbi09,reuter-ijcv10,sharma-horaud},
it would help to have an \emph{a priori} estimate
for the number of eigenfunctions necessary to 
embed the hippocampus by its eigenfunctions into Euclidean space.
As the hippocampus is initially embedded in Euclidean space, 
the reason for re-embedding it by its eigenfunctions is geometric,
as explained in Section \ref{S:implications} below. 
The 3D images $\Phi^3 : M \to \real^3$ of a few human model surfaces
are plotted in Figure \ref{figure:humanembeddings}.  
From this figure, one may get a sense of why eigenfunction
embeddings have been used to find point correspondences between shapes,
as the arms and legs are better aligned in the image.
The eigenfunctions in these examples are computed 
using the normalized graph Laplacian with Gaussian weights 
(cf.\ \cite{vonluxburg-consistency,belkin-niyogi-convergence,
ting2011analysis} and references therein).

%----------------------------------------
\begin{figure}[ht]
\vspace{2em}
\begin{center}
\begin{tabular}{cccc}        
\includegraphics[height=0.9in]{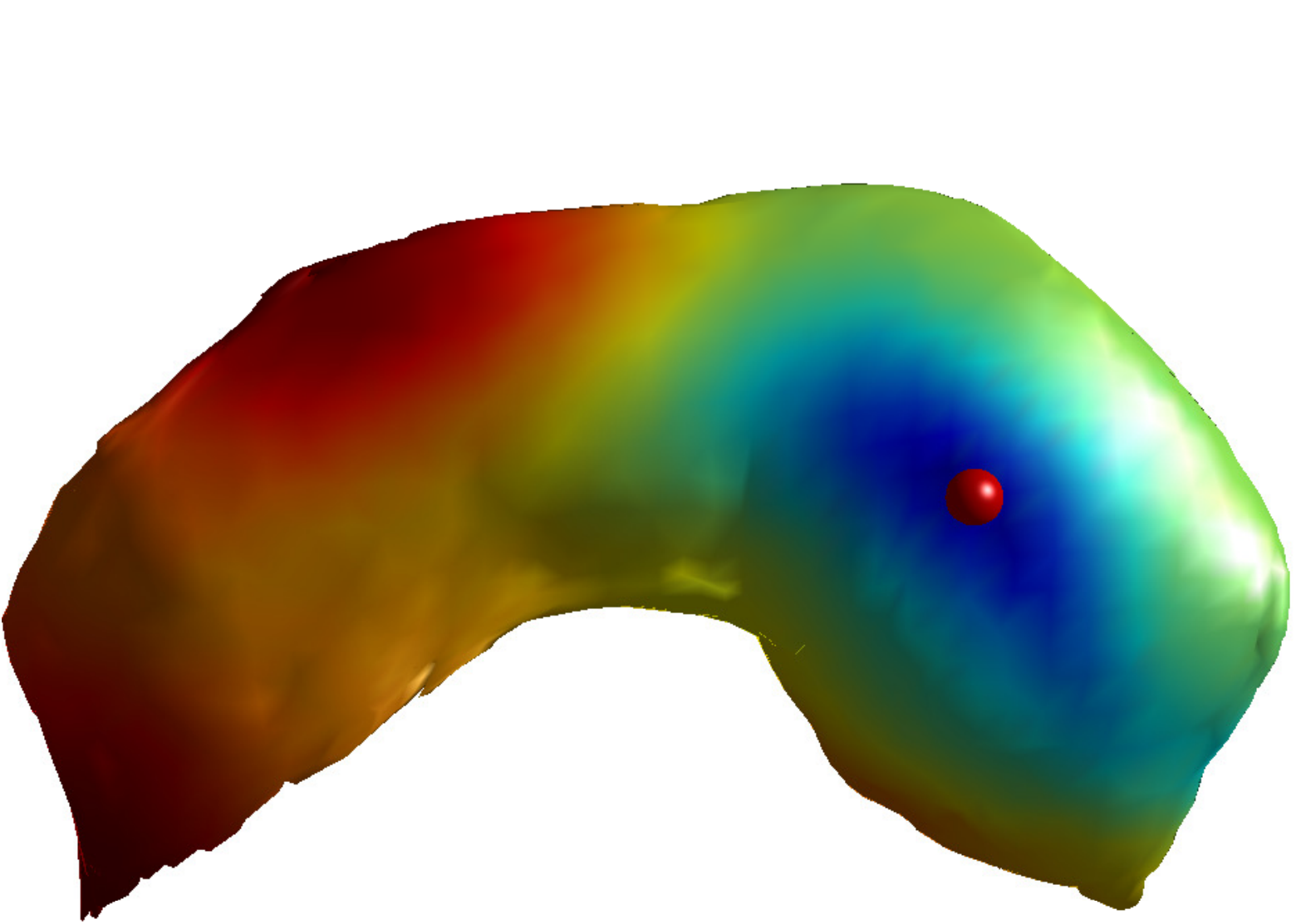} & \hspace{0.8em}
\includegraphics[height=0.78in]{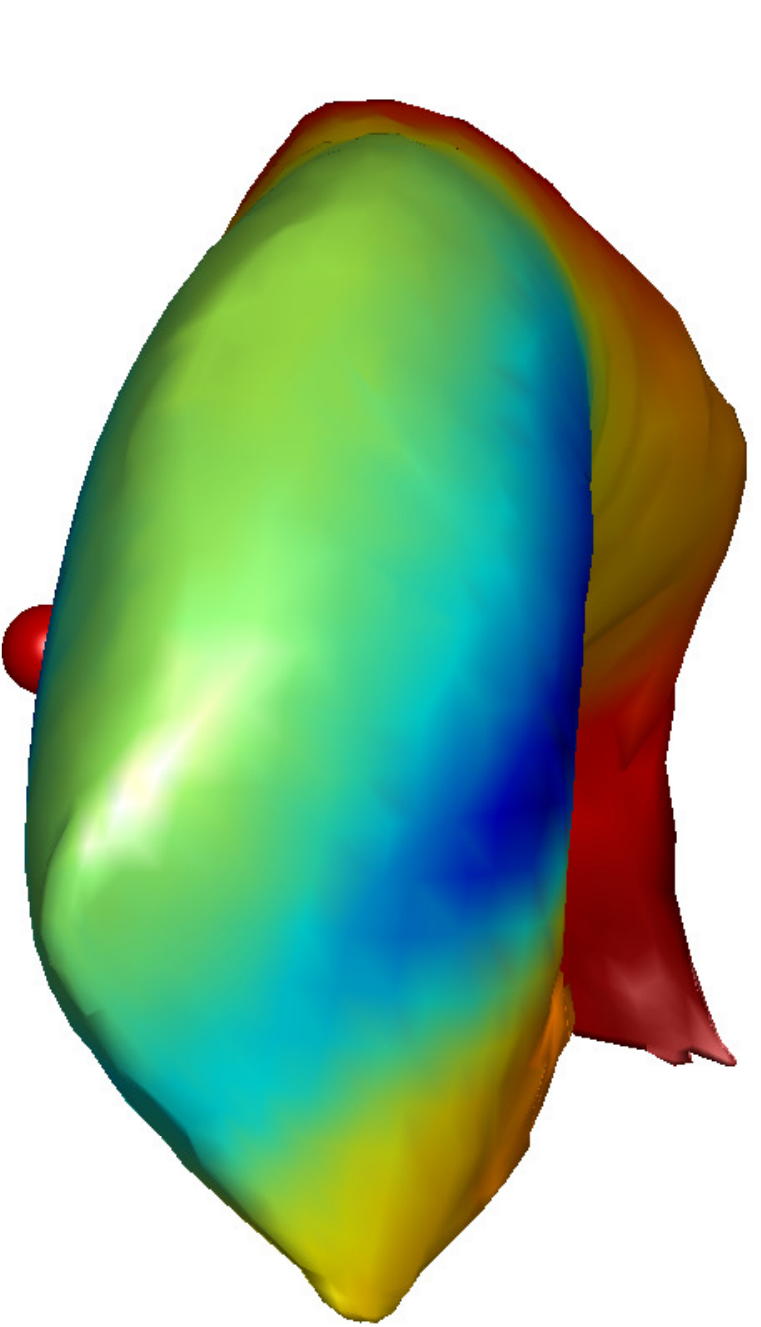} & \hspace{4em}
\includegraphics[height=0.9in]{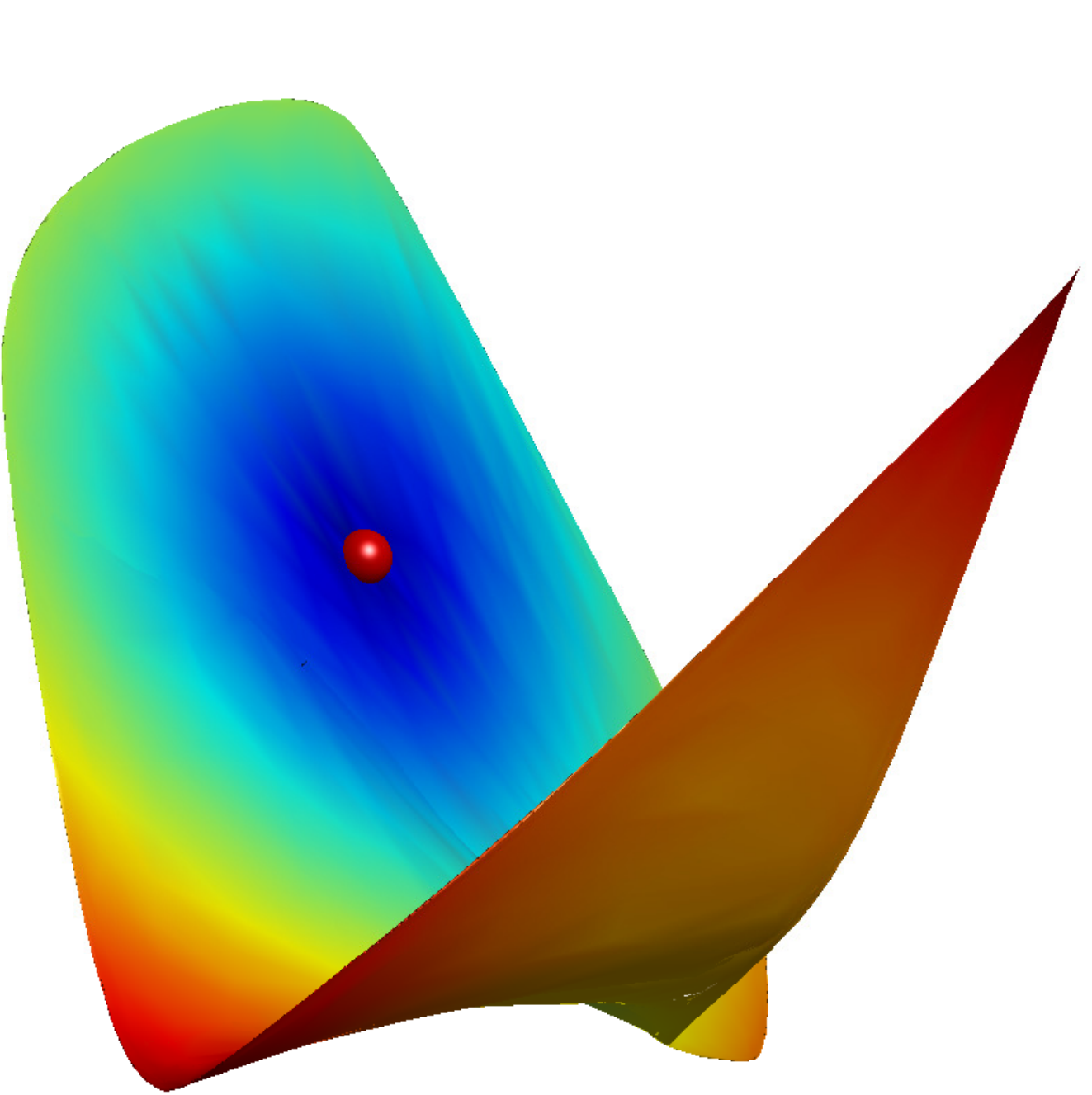} & \hspace{0.5em}
\includegraphics[height=0.9in]{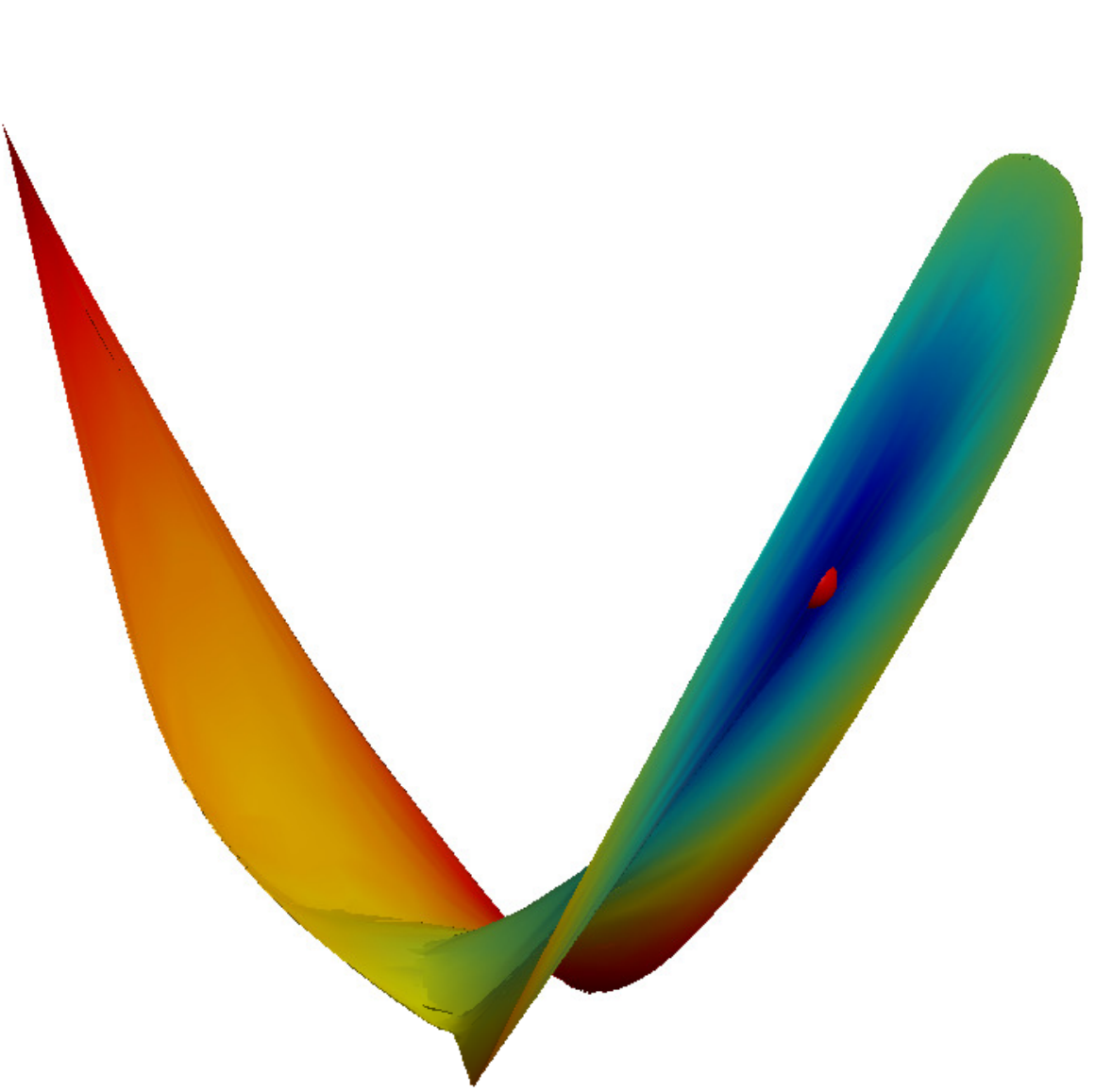}
\end{tabular}	
\end{center}
\caption{A hippocampus from two angles (left)
and its 3D eigenfunction map (right). Surface color is given by 
distance in spectral space from the point indicated by the ball.}
\label{figure:hippocampus_registration}
\end{figure}
%------------------------------

\begin{figure}[ht]
\begin{center}
\vspace{2em}
\begin{tabular}{cc}
\includegraphics[height=1.2in]{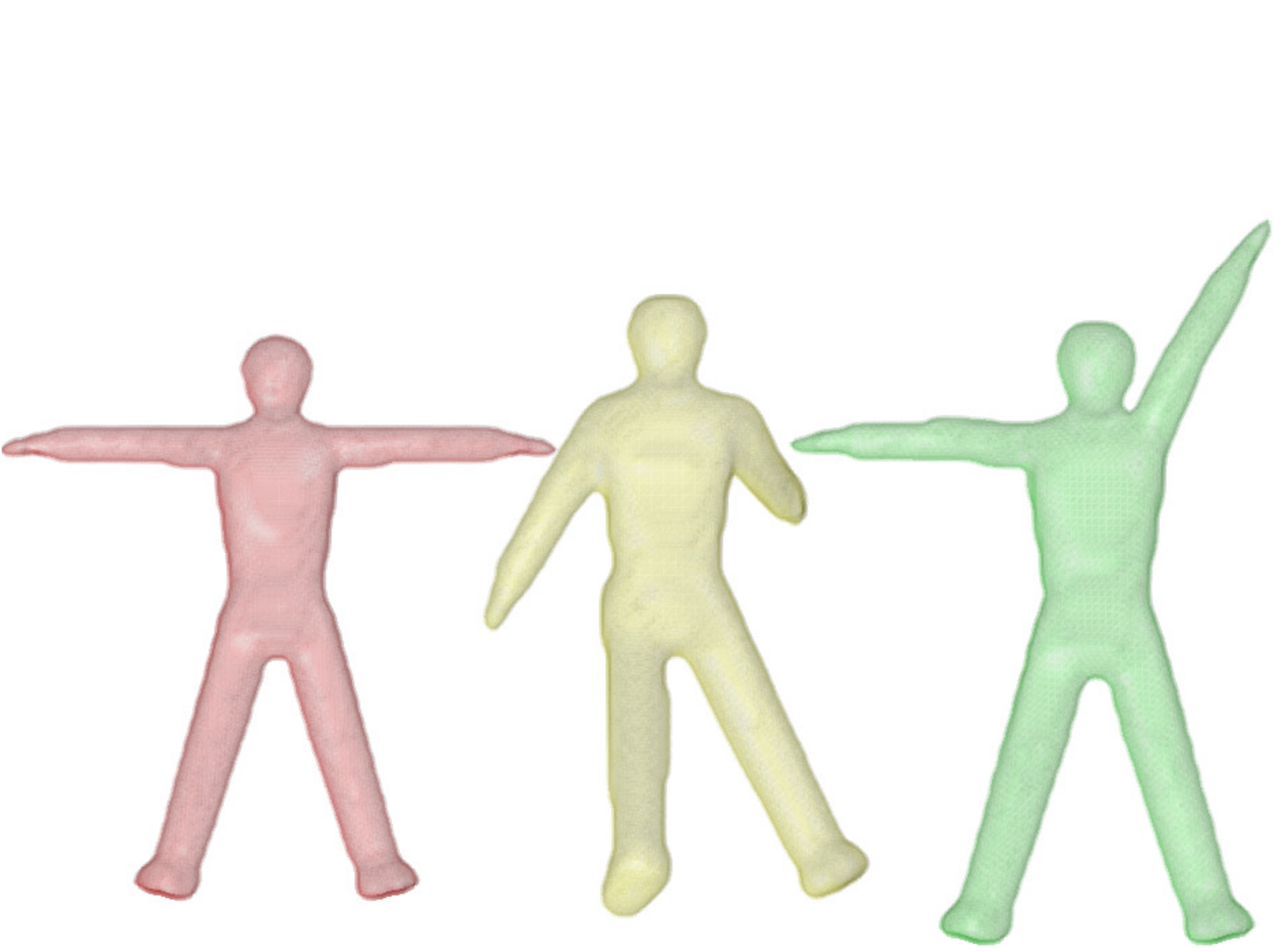} 
\hspace{2em} & \hspace{2em}
\includegraphics[height=1.2in]{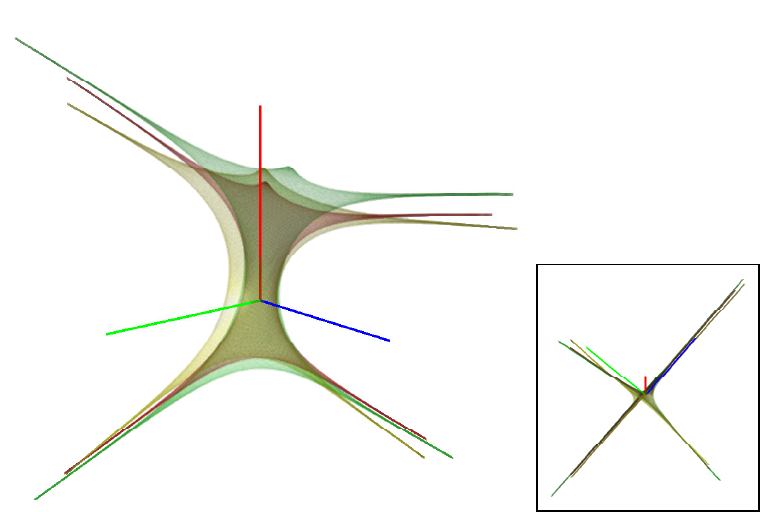} 
\end{tabular}
\end{center}
\caption{A few human model surfaces (left) and their 3D eigenfunction maps (right).
Two angles of the image are shown. Note that axes are also plotted.}
\label{figure:humanembeddings}
\end{figure}

We now recall some relevant notions from differential geometry.
Let $M, M'$ be smooth manifolds.
A smooth map $F : M \to M'$ is called an \emph{immersion} if $\rank dF_x = \dim M$
for every $x \in M$.  A smooth map $F : M \to M'$ is called a \emph{(smooth) embedding} if 
$F$ is an immersion and a homeomorphism onto its image $F(M)$. 
Recall that for a compact manifold $M$,
if $F : M \to M'$ is an injective immersion, then it is a smooth embedding.

Suppose now that $M = (M,g)$ and $M' = (M',g')$ are Riemannian manifolds.
We write the corresponding geodesic distance metrics as $d_M$ and $d_{M'}$.
For $M$ and $M'$ to be \emph{isometric} 
means that there is a diffeomorphism $F : M \to M'$ such that $F^\ast g' = g$. 
Such a map $F : M \to M'$ is called an \emph{isometry}. In particular, if $F : M \to M'$ 
is an isometry, then $d_M(x,y) = d_{M'}(F(x),F(y))$ for all $x, y \in M$.

Let $M = (M,g)$ be a complete $n$-dimensional Riemannian manifold.
Herein, $B(x,r)$ will denote the geodesic ball of radius $r$ centered at $x \in M$,
and $B(r)$ will denote the Euclidean ball of radius $r$ centered at the origin of $\real^n$.
As $M$ is complete, the domain of the exponential map
is $T_xM \cong \real^n$, i.e.\ $\exp_x : \real^n \to M$.
The \emph{injectivity radius} of $M$, denoted $\inj(M)$,
is the largest real number for which the restriction
$\exp_x : B(r) \subseteq \real^n \to B(x,r)$
is a diffeomorphism for all $x \in M$, $r \leq \inj(M)$.

Let $x \in M$, and let $P$ be a 2-plane in $T_xM$.   
The circle of radius $r < \inj(M)$ centered at $0$ in $P$
is mapped by $\exp_x : \real^n \to M$ to the geodesic circle 
$C_P(r)$, whose length we denote $l_P(r)$. 
Then
\begin{equation}
l_P(r) = 2\pi r \, (1-\frac{r^2}{6} K(P) + O(r^3) ) \quad \text{as $r \to 0^+$}.
\end{equation}
The number $K(P)$ is called the \emph{sectional curvature} of $P$.
If $\dim M = 2$, then $K(x) = K(T_xM)$ is equivalent to the 
Gaussian curvature at $x$.

Next, we use $\vol$ to denote the canonical Riemannian measure associated with $(M,g)$.
Let $x \in M$. The pulled-back measure $\exp_x^\ast(\vol)$
has a density with respect to the Lebesgue measure in $T_xM \cong \real^n$.
Let $(r,u) \in [0,\infty) \times S^{n-1}$ be polar coordinates in $T_xM$.
For $r < \inj(M)$, we may write $\exp_x^\ast(V) = \theta_x(r,u) \, dr \, du$.
Then
\begin{equation}
\theta_x(r,u) = r^{n-1} (1-\frac{r^2}{6} \ricci_x(u,u) + O(r^3) ) 
\quad \text{as $r \to 0^+$}.
\end{equation}
The term $\ricci_x(u,u)$ is a quadratic form in $u$, whose associated symmetric bilinear form is called the \emph{Ricci curvature} at $x$. 
If $\dim M = 2$, then $\ricci_x(u,u) = K(x) g(u,u)$, where $K(x)$ is 
the Gaussian curvature at $x$.

Heat flow on a closed Riemannian manifold $(M,g)$
is modeled by the heat equation
\begin{equation}
(\partial_t + \Delta) u(t,x) = 0\,,
\label{equation:heat}
\end{equation}
where $\Delta$ is the Laplacian of $M$ applied to $x \in M$.
Any initial distribution $f \in L^2(M)$ 
determines a unique smooth solution $u(t,x)$, $t>0$, 
to \eqref{equation:heat} such that  $u_t \to_{L^2} f$
as $t \to 0^+$.  This solution is given by
\begin{equation}
u(t,x) = \int_M p(t,x,y) f(y) \, \dvol(y),
\end{equation}
where $p \in C^\infty( \real^+ \times M \times M )$ 
is called the \emph{heat kernel} of $M$. For example,
the heat kernel of $\real^n$ (with Euclidean metric) is the familiar Gaussian kernel.
Lastly, the heat kernel may be expressed in the eigenvalues-functions as
\begin{equation}
p(t,x,y) = \sum_{j=0}^\infty \, e^{-\lambda_j t} \varphi_j(x) \varphi_j(y).
\label{equation:hksum}
\end{equation}
For more on the Laplacian, heat kernel, and Riemannian geometry, 
we refer the reader to, e.g., \cite{berard86,chavel84,rosenberg97,grigoryan09}.

We are now ready to state the results of this note.
Let $\kappa_0 \geq 0, i_0 > 0$ be fixed constants, $n \geq 2$, and 
consider the class of closed, connected $n$-dimensional Riemannian manifolds
\begin{equation}
\label{E:mathcalm}
\begin{split}
\mathcal{M} := \{ \; (M,g) \mid \;
& \dim M = n, \; \ricci_M \geq - (n-1)\kappa_0 g, \\
& \inj(M) \geq i_0, \; \vol(M) = 1 \; \} \,.
\end{split}
\end{equation}
Note that $\ricci_M \geq - (n-1) \kappa_0 g$ means
\begin{equation}
\ricci(\xi,\xi) \geq -(n-1) \kappa_0 \,g(\xi,\xi) \qquad (\forall \,\xi \in TM).
\end{equation}
If $M$ is a surface and $K$ denotes its Gaussian curvature,
then $\ricci_M \geq - (n-1)\kappa_0 g$ is equivalent to $K \geq -\kappa_0$.

Note that the following Theorems \ref{T:localed}, \ref{T:ued}, and \ref{T:ueds}
are independent of the choice of eigenfunction basis.
We first show that the eigenfunction maps $\Phi^m$ are well-controlled immersions
in the sense that the neighborhoods on which they are embeddings cannot
be too small.  
\begin{theorem}
\label{T:localed}
There is a positive integer $m$ and constant $\epsilon > 0$
such that, for any $M \in \mathcal{M}$, for all $z \in M$, 
\begin{align*}
\Phi^m_z : B(z,\epsilon) &\longrightarrow \real^m \\
x &\longmapsto (\varphi_{1}(x), \dotsc, \varphi_{m}(x))
\end{align*}
is a smooth embedding.
\end{theorem}
The proofs are deferred to the sections following.
Our main goal is to prove the following result.
\begin{theorem}[Uniform maximal embedding dimension]
\label{T:ued}
There is a positive integer $d$ such that, for all $M \in \mathcal{M}$, 
\begin{align*}
\Phi^d : M &\longrightarrow \real^d \\
x &\longmapsto (\varphi_1(x),\dotsc,\varphi_d(x))
\end{align*}
is a smooth embedding.
\end{theorem}

We lastly consider closed, connected surfaces isometrically immersed in $\real^3$.
We denote mean curvature by $H$, Gaussian curvature by $K$,
and surface area by $\vol$.
Let $H_0, \kappa_0, A$ be fixed positive constants and 
consider the class of surfaces
\begin{equation}
\begin{split}
\mathcal{S} := \{ \; (M,g) \mid \;
& \dim M = 2, \; \abs{K} \leq \kappa_0, \\
& \abs{H} \leq H_0, \; \vol(M) \leq A, \\ 
& \iota : M \hookrightarrow \real^3\ \text{is an isometric immersion} \;\} \,.
\end{split}
\end{equation}

\begin{theorem}[Uniform maximal embedding dimension for surfaces]
\label{T:ueds}
There is a positive integer $d$ such that, for all $M \in \mathcal{S}$, 
\begin{align*}
\Phi^d : M &\longrightarrow \real^d \\
x &\longmapsto (\varphi_1(x),\dotsc,\varphi_d(x))
\end{align*}
is a smooth embedding.
\end{theorem}

Before continuing, we consider the natural question 
of whether the eigenfunction maps are stable under perturbations of the metric.
This has been answered in \cite{bbg94}.
\begin{theorem}[B\'erard-Besson-Gallot \cite{bbg94}]
\label{T:stability1}
Let $(M,g)$ be a closed $n$-dimensional Riemannian manifold, $\epsilon_0 >0$,
and $m \in \naturals$. Let $g'$ be any metric on $M$
such that $(1-\epsilon) g \leq g' \leq (1+\epsilon) g$, $\epsilon \in [0, \epsilon_0)$.
We assume that all metrics under consideration satisfy
$\ricci_{(M,g')} \geq - (n-1)\kappa_0 g'$ for some constant $\kappa_0 \geq 0$.
There exist constants $\eta_{g,j,\kappa_0}(\epsilon), 1 \leq j \leq m$,
which go to 0 with $\epsilon$, such that to any orthonormal basis
$\{ \varphi'_j \}$ of eigenfunctions of $\Delta_{g'}$ one can
associate an orthonormal basis $\{ \varphi_j \}$ of eigenfunctions 
of $\Delta_g$ satisfying 
$\norm{\varphi'_j-\varphi_j}_{L^\infty} \leq \eta_{g,j,\kappa_0}(\epsilon)$
for $j \leq m$.
\end{theorem}

% % % % % % % % % % % % % % % % % % % % % % % % % % % % % % 
\subsection{Motivations from eigenfunction-based shape registration methods}
\label{S:implications}

Here we consider the significance of a uniform maximal embedding dimension 
from the perspective of the shape registration methods in \cite{jain-correspondence07,
mateus-etal08,liu-icpr08,bates-isbi09,reuter-ijcv10,sharma-horaud}.
In shape registration, we begin with two closed, connected Riemannian manifolds
$M = (M,g)$ and $M' = (M',g')$, and our goal is to find a correspondence 
between them given by $\alpha : M \to M'$. 
(Note some use a looser notion of correspondence, e.g.\ \cite{memoli-spectral},
allowing for many-many matches between points of the ``shapes''.)
Moreover, if $M$ and $M'$ are isometric,
we require the correspondence $\alpha : M \to M'$ to be an isometry.
This correspondence may be established using eigenfunction maps, 
followed by closest point matching as follows.
Here we must be precise regarding the choice of eigenfunction basis, 
and we let $\mathcal{B}(M)$ denote the set of orthonormal
bases of real eigenfunctions of the Laplacian of $M$.
For $m \in \naturals$ and $b \in \mathcal{B}(M)$, $b = \{ \varphi_j^b \}_{j \in \naturals}$,
let $\Phi^m_b$ denote the corresponding eigenfunction map, i.e.\ 
$x \mapsto \{ \varphi_j^b(x) \}_{1 \leq j \leq m}$.
Given $b \in \mathcal{B}(M)$, $b' \in \mathcal{B}(M')$, and $m \in \naturals$,
we consider as a potential correspondence the map $\alpha(b,b',m) : M \to M'$ given by
\begin{equation}
\label{E:defalpham}
\alpha(x; b,b',m) := \arg \inf_{x' \in M'} \; 
\norm{\, \Phi^m_{b'}(x') - \Phi^m_b(x) \,}_{\real^m} \,,
\end{equation}
ties being broken arbitrarily.  
We first consider the sense in which $\alpha$ yields the desired correspondence
for isometric shapes,
and then the sense in which $\alpha$ is stable.

\begin{proposition}
If $M$ and $M'$ are isometric and $m \geq$ the maximal embedding dimensions
of $M$ and $M'$, 
then $\alpha(b,b',m) : M \to M'$ is an isometry for some choice of 
$b \in \mathcal{B}(M), b' \in \mathcal{B}(M')$.
\label{T:isomalpha}
\end{proposition}
\begin{proof}
Let $F : M \to M'$ be an isometry, and let $m \geq$ the maximal embedding dimensions
of $M$ and $M'$. 
Note that there are $b \in \mathcal{B}(M), b' \in \mathcal{B}(M')$ 
such that $\varphi_j^b = \varphi_j^{b'} \circ F$ for all $j \in \naturals$ (cf.\ \cite{chavel84}).
In particular, $\Phi^m_b(x)=\Phi^m_{b'}(F(x))$ for all $x \in M$.
Since $\Phi^m_{b'}$ is injective (as it is an embedding), the infimum in \eqref{E:defalpham}
is uniquely realized for each $x \in M$.  Hence $\alpha(b,b',m) \equiv F$.
\end{proof}

Now let $M = (M,g)$ be any closed, connected Riemannian manifold, $\epsilon_0 > 0$ fixed, 
and $g_{\epsilon}$, $\epsilon \in [0, \epsilon_0)$,
a family of Riemannian metrics on $M$ such that
$(1-\epsilon) g \leq g_{\epsilon} \leq (1+\epsilon) g$
for all $\epsilon \in [0, \epsilon_0)$. 
We assume that there exist $\kappa_0 \geq 0, i_0 > 0$ for which, 
with $\mathcal{M}$ as defined in \eqref{E:mathcalm},
$M_{\epsilon} := (M,g_{\epsilon}) \in \mathcal{M}$ for all $\epsilon \in [0, \epsilon_0)$.
For each $\epsilon \in [0,\epsilon_0)$, let $b'_{\epsilon} \in \mathcal{B}(M_{\epsilon})$
be arbitrary. 
The following proposition is an immediate consequence of Theorem \ref{T:stability1}, 
the triangle inequality, and the definition of $\alpha$.
\begin{proposition}
Let $m \in \naturals$. 
There is a constant $\eta_m(\epsilon)$, which goes to 0 with $\epsilon$, 
and $b : \epsilon \in [0,\epsilon_0) \mapsto b_{\epsilon} \in \mathcal{B}(M)$
such that, for all $\epsilon \in [0, \epsilon_0)$,
\begin{equation}
\sup_{x \in M} \;
\norm{\,\Phi^m_{b'_{\epsilon}}(\alpha(x; b_{\epsilon},b'_{\epsilon},m)) 
- \Phi^m_{b'_{\epsilon}}(x)\,}_{\real^m} 
\leq \eta_m(\epsilon) \,,
\end{equation}
where $\alpha(b_{\epsilon},b'_{\epsilon},m)$ is defined as in \eqref{E:defalpham}.
\end{proposition}

The size of the search space of potential correspondences
$\{ \, \alpha(b,b',m) \mid b \in \mathcal{B}(M), b'\in \mathcal{B}(M') \, \}$
grows at least exponentially in $m$.
To see this, note that we may arbitrarily flip the sign of any eigenfunction, and
so $\abs{ \{\, \Phi^m_b \mid b \in \mathcal{B}(M) \,\} } \geq 2^m$.
Consequently, to find the isometry asserted by Proposition \ref{T:isomalpha} 
with minimal computational demands, 
it would be useful to know the maximal embedding dimensions
of $M$ and $M'$.

% % % % % % % % % % % % % % % % % % % % % % % % % % % % % % 
\subsection{Examples: the embedding dimensions of the sphere and stretched torus}

We now compute the embedding dimensions of the standard sphere
and a ``stretched torus'' using formulas for their eigenfunctions.
One usually cannot derive the embedding dimension in this way, however, 
as, to paraphrase from \cite{zelditch-eigenfunctions}, 
there are only a few Riemannian manifolds for which we 
have explicit formulas for the eigenfunctions.  

Identifying the standard sphere $S^n = (S^n,\can)$ with the Riemannian submanifold
\begin{equation}
\{ \; (x^1,\dotsc,x^{n+1}) \mid \norm{x}_{\real^{n+1}} = 1 \; \} 
\end{equation}
of $\real^{n+1}$, the eigenfunctions of $\Delta_{S^n}$ 
are restrictions of harmonic homogeneous polynomials on $\real^{n+1}$ 
\cite{zelditch-eigenfunctions,chavel84}.
A polynomial $P(x^1,\dotsc,x^{n+1})$ on $\real^{n+1}$ is called 
(1) \emph{homogeneous (of degree $k$)} if $P(r x) = r^k P(x)$
and (2) \emph{harmonic} if $\Delta_{\real^{n+1}} P(x) = 0$.
Moreover, if $P(x)$ is a harmonic homogeneous polynomial of degree $k$,
then its corresponding eigenvalue is $\lambda = k(n+k-1)$, whose
multiplicity is
\begin{equation}
\begin{pmatrix} n+k \\ k \end{pmatrix} - \begin{pmatrix} n+k-2 \\ k-2\end{pmatrix} \,.
\end{equation}
One may show that an $L^2(S^n)$-orthogonal basis
of the eigenspace corresponding to $\lambda(S^n) = n$ is given 
by the restriction of the coordinate functions $x^1,\dotsc,x^{n+1}$ on $\real^{n+1}$
to $S^n$ (cf.\ Proposition 1, p.\ 35, \cite{chavel84}). We immediately have
\begin{proposition}
The embedding dimension of $S^n$ is $d = n+1$.
\end{proposition}
Although we get an explicit answer for the sphere, 
it does not reveal how geometry influences the embedding dimension.
Let us look at another space.

Explicit formulas are also available for the eigenfunctions of products of spheres, e.g.\ tori,
by virtue of the decomposition $\Delta_{M \times N} = \Delta_M + \Delta_N$.
We consider stretching a flat torus to have a given injectivity radius 
and volume, and then explicitly compute the embedding dimension.
We see that the embedding dimension depends on both injectivity radius and volume,
and thus cannot be bounded using only curvature and volume bounds, 
or curvature and injectivity radius
bounds.  In particular, let $0 < a < b$, $n \geq 2$, and consider the flat $n$-torus $T$
constructed by gluing the rectangle 
\begin{equation}
\{ \; (x^1,\dotsc,x^n) \mid 
0 \leq x^j \leq a \; (j \neq n), \; 
0 \leq x^n \leq b \; \} 
\end{equation}
as usual. Note
$\ricci_{T} = 0$, $\inj(T) = a/2$, and $\vol(T)=a^{n-1}b$.
\begin{proposition}
The embedding dimension of $T$ is 
\begin{equation}
\begin{split}
d &= 2( \lceil a^{-1} b \rceil + n-2) \\
&\geq 2^{1-n} \vol(T) / \inj(T)^n \,,
\end{split}
\end{equation}
where $\lceil x \rceil = $ the smallest integer greater than or equal to $x$.
\end{proposition}
\begin{proof}
Put $f_1(x) := \cos(2\pi x)$, $f_2(x) := \sin (2\pi x)$.
The unnormalized real eigenfunctions of $T$ are
\begin{equation}
f_{k_1}(a^{-1} m_1 x^1) \dotsm f_{k_{n-1}}(a^{-1} m_{n-1} x^{n-1})
f_{k_n}(b^{-1} m_n x^n) \qquad (m_i \in \naturals,\ k_i \in \{1,2 \}),
\end{equation}
with corresponding eigenvalues 
\begin{equation}
\lambda(m_1,\dotsc,m_n) = (2\pi)^2
(a^{-2} m_1^2 + \dotsb+ a^{-2} m_{n-1}^2 + b^{-2} m_n^2).
\end{equation}
We denote $\lambda(m_j,j) = \lambda(0,\dotsc,m_j,\dotsc,0)$.

First, suppose $a^{-1} b$ is not an integer, and put $p := \lfloor a^{-1} b \rfloor$.
One may check that the initial sequence of eigenvalues corresponds to 
\begin{equation}
\label{E:eigenvalueordering1}
\begin{split}
0 &< \lambda(1,n) 
< \lambda(2,n) 
< \dotsb
< \lambda(p,n) \\
& < \lambda(1,1) 
= \lambda(1,2) 
= \dotsb
=  \lambda(1,n-1) \\
& < \dotsb \,.
\end{split}
\end{equation}
The eigenvalues $\lambda(k,n)$, $k \leq p$, each have multiplicity 2; 
for example, the eigenspace corresponding to $\lambda(k,n)$ has as a basis 
$\{ f_{1}(b^{-1} k x^n), f_{2}(b^{-1} k x^n) \}$. 
It follows that $\Phi^{2p} : T \to \real^{2p}$ depends only on $x^n$.
It is readily verified
that $x^n \mapsto \Phi^{2}(x)$ is injective since,
up to phase, $\Phi^{2}(x) = (f_{1}(b^{-1} x^n), f_{2}(b^{-1} x^n))$.
Thus $x^n \mapsto \Phi^{2p}(x)$ is injective.
Put $F(x^j) = (f_{1}(a^{-1} x^j), f_{2}(a^{-1} x^j))$. Then, up to phase
and up to a permutation of the last $2(n-1)$ coordinates,
\begin{equation}
\Phi^{2p+2(n-1)}(x) = (\Phi^{2p}(x), F(x^1), F(x^2), \dotsc, F(x^{n-1})) \,.
\end{equation}
Noting $x^j \mapsto F(x^j)$ is an embedding of $[0,a]/(0 \sim a)$ into $\real^2$,
we deduce that $\Phi^{2p+2(n-1)} : T \to \real^{2p+2(n-1)}$
is an embedding and, furthermore,
that if any one of the last $2(n-1)$ coordinates are removed,
then the map is no longer injective.  It follows that $d=2p+2(n-1) 
=2( \lceil a^{-1} b \rceil + n-2)$
is the embedding dimension
of $T$ when $a^{-1} b$ is not an integer.

Now suppose that $a^{-1} b$ is an integer; put $p := a^{-1} b$.
One may check that the initial sequence of eigenvalues is 
\begin{equation}
\label{E:eigenvalueordering2}
\begin{split}
0 &< \lambda(1,n) 
< \dotsb
< \lambda(p-1,n) \\
&< (2\pi)^2 a^{-2} = \lambda(1,1) = \dotsb = \lambda(1,n-1) = \lambda(p,n) \\
&< \dotsb \,.
\end{split}
\end{equation}
Following the preceding arguments, we see that
$\Phi^{2(p-1)+2(n-1)} : T \to \real^{2(p-1)+2(n-1)}$
is an embedding when the eigenfunctions are ordered 
according to the sequence suggested by \eqref{E:eigenvalueordering2},
where the two eigenfunctions corresponding to $\lambda(p,n)$
are not included as coordinates.
\end{proof}

\begin{remark}
Note the stretched torus example shows that the embedding dimension of 
$\mathcal{M}(n,\kappa_0,i_0)$ is bounded below by 
$2^{1-n} i_0^{-n}$.
\end{remark}

%%%%%%%%%%%%%%%%%%%%%%%%%%%%%%%%%%%%%%%%%%%%%%%%%%%%%%%%%%%%%%%%%%%%%%%%%%%%%%%
% SECTION : Proof 1
%%%%%%%%%%%%%%%%%%%%%%%%%%%%%%%%%%%%%%%%%%%%%%%%%%%%%%%%%%%%%%%%%%%%%%%%%%%%%%%

\section{Proof of Theorem \ref{T:localed}}

We first show that the manifolds of $\mathcal{M}$ have uniformly bounded diameter.
That is, there is a $D > 0$ such that diameter $\diam(M) \leq D$ for all $M \in \mathcal{M}$. 
Recall $d(M) := \sup_{x,y \in M} d_M(x,y)$.
To see this, let $M \in \mathcal{M}$. By the Theorem of Hopf-Rinow,
we may take a unit speed geodesic $\gamma : \real \to M$ that realizes the diameter,
say, $d(\gamma(0),\gamma(\diam(M))) = \diam(M)$. Stack geodesic balls of radius
$i_0/2$ end-to-end along $\gamma$. 
It is a simple exercise in proof by contradiction to show these balls
are disjoint. The volumes of these balls are uniformly bounded below by
Croke's estimate (see below).  Finally,
the volume requirement $\vol(M) = 1$ limits the number of such balls,
hence the diameter of $M$.

We now recall a few function norms (cf., e.g., \cite{evanspde}).
Let $\Omega \subseteq \real^n$ be open, $0 <\alpha \leq 1$, 
$k$ a nonnegative integer, $1 \leq p < \infty$.
In this note, the following norms and seminorms will be used with 
a smooth function $f : \Omega \to \real$. We write
\begin{align}
\norm{f}_{C(\bar{\Omega})} &:= \sup_{x \in \Omega} \; \abs{f(x)} \\
[f]_{C^{\alpha}(\bar{\Omega})} &:= \sup_{x,y \in \Omega, \,x \neq y} \, 
\frac{\abs{f(x)-f(y)}}{\norm{x-y}_{\real^n}^{\alpha}} \\
\norm{f}_{C^{\alpha}(\bar{\Omega})} &:=
\norm{f}_{C(\bar{\Omega})} + [f]_{C^{\alpha}(\bar{\Omega})} \\
\norm{f}_{W^{k,p}(\Omega)} &:= 
\bigg( \sum_{\abs{\alpha} \leq k} 
\int_{\Omega} \, \abs{D^{\alpha} f}^p \, dx \bigg)^{1/p} \,.
\end{align}

Theorem \ref{T:localed} is an adaptation of the following local embedding result.
\begin{theorem}[Jones-Maggioni-Schul \cite{jms08}; see also \cite{jms10}]
\label{T:jms}
Assume $\vol(M) = 1$. Let $z \in M$ and suppose $u : U \to \real^n$ 
is a chart satisfying the following properties. \par
There exist positive constants $r, C_1, C_2$ such that \\
(1) $u(z) = 0$; \\
(2) $u(U) = B$, where $B := B(r)$ is the ball of radius $r$ in $\real^n$ 
centered at the origin; \\
(3) for some $\alpha >0$, the coefficients $g^{ij}(u) = g(du^i, du^j)$ 
of the metric inverse satisfy $g^{ij}(0) = \delta^{ij}$ and
are controlled in the $C^{\alpha}$ topology on $B$:
\begin{align}
C_1^{-1} \norm{\xi}_{\real^n}^2 \leq 
&\sum_{ij} \, \xi_i \xi_j g^{ij}(u) \leq C_1 \norm{\xi}_{\real^n}^2 
&&(\forall \, u \in B, \, \forall \, \xi \in \real^n); \\
&[g^{ij}]_{C^{\alpha}} \leq C_2 
&&(\forall \, i,j).
\end{align} \par
Then there are constants $\nu = \nu(n, C_1, C_2) > 1$,
$a_j > 0, j = 1, \dotsc, n$, and integers $j_1, \dotsc, j_n$
such that the following hold. \\
(a) The map
\begin{align*}
\Phi_{a} : B(z,\nu^{-1} r) &\longrightarrow \real^n \\
x &\longmapsto (a_1 \varphi_{j_1}(x), \dotsc, a_n \varphi_{j_n}(x))
\end{align*}
satisfies, for all $x, y \in B(z,\nu^{-1} r)$,
\begin{equation}
\frac{\nu^{-1}}{r} \, d_M(x,y) 
\leq \norm{\Phi_{a}(x)-\Phi_{a}(y)}_{\real^n}
\leq \frac{\nu}{r} \, d_M(x,y) \,;
\end{equation} 
(b) the associated eigenvalues satisfy 
$\nu^{-1} r^{-2} \leq \lambda_{j_1}, \dotsc, \lambda_{j_n} \leq \nu r^{-2}$.
\end{theorem}
We point out that this result (Theorem 2.2.1 in \cite{jms10}) is 
stated for $g \in C^{\alpha}, \alpha > 0,$ and $M$ possibly having a boundary.
We now invoke an eigenvalue bound to use with (b) in Theorem \ref{T:jms}.
\begin{theorem}[B\'erard-Besson-Gallot \cite{bbg94}]
\label{T:eigenvaluebounds}
Let $M$ be a closed, connected Riemannian manifold such that
$\dim M = n$, $\ricci_M \geq -(n-1)\kappa_0 g$, and $\diam(M) \leq D$.
There is a constant $C_{\lambda} = C_{\lambda}(n,\kappa_0,D)$ such that 
\begin{equation*}
C_{\lambda} \,j^{2/n} \leq \lambda_j(M) 
\qquad (\forall \, j\geq 0).
\end{equation*}
\end{theorem}

Finally, we must choose a coordinate system
satisfying the hypotheses of Theorem \ref{T:jms}.
We use harmonic coordinates.  
By definition, a coordinate chart $(U,x^i)$ of $M = (M^n,g)$ is \emph{harmonic}
if $\Delta_M x^i=0$ on $U$ for $i=1,\dotsc,n$ (cf., e.g., \cite{sabitov-shefel,kazdan-deturck}).
All necessary properties of harmonic coordinates for this note are 
contained in the following result, which follows from the proof of
Theorem 0.3 in Anderson-Cheeger \cite{anderson-cheeger}.
\begin{lemma}
\label{T:ac}
Let $\kappa_0 \geq 0$ and $i_0 > 0$,
let $(M,g)$ be a closed $n$-dimensional Riemannian manifold satisfying
\begin{equation}
\ricci_M \geq - (n-1) \kappa_0 g, \quad
\inj(M) \geq i_0,
\end{equation}
and let $\alpha \in (0,1)$ and $Q>1$ be fixed. 
Then there exist
constants $r_h, C_h$, both depending only on $n, \kappa_0, i_0, \alpha, Q$,
such that for all $z \in M$
there is a harmonic coordinate chart $u : U \to \real^n$ satisfying \\
(1) $u(z) = 0$; \\
(2) $u(U) = B$, where $B := B(r_h)$ is the ball of radius $r_h$ in $\real^n$ 
centered at the origin; \\
(3) the coefficients $g^{ij}(u) = g(du^i, du^j)$ of the metric inverse
satisfy $g^{ij}(0) = \delta^{ij}$ and
are controlled in the $C^{\alpha}$ topology on $B$:
\begin{align}
Q^{-1} \norm{\xi}_{\real^n}^2 \leq &\sum_{ij} \, \xi_i \xi_j g^{ij}(u) 
\leq Q \norm{\xi}_{\real^n}^2 
&& (\forall \, u \in B, \, \forall \, \xi \in \real^n); \label{E:invmetricbound} \\
& [g^{ij}]_{C^{\alpha}} \leq C_h 
&& (\forall \, i,j).
\end{align}
\end{lemma}
In deriving Lemma $\ref{T:ac}$,
we will use the following Sobolev-type estimate
(cf.\ Theorem 5.6.5 in Evans \cite{evanspde}).
\begin{proposition}[Morrey's inequality]
Let $\Omega \subset \real^n$ be open, bounded, and with $C^1$ boundary.
Assume $p > n$ and $u \in W^{1,p}(\Omega)$ is continuous.
Then $u \in C^{\alpha}(\bar{\Omega})$, for $\alpha = 1-n/p$, with
\begin{equation}
\norm{u}_{C^{\alpha}(\bar{\Omega})} 
\leq C \norm{u}_{W^{1,p}(\Omega)} \,,
\end{equation}
where $C$ is a constant depending only on $n,\alpha, \Omega$.
\end{proposition}
\begin{proof}[Proof of Lemma \ref{T:ac}]
Theorem 0.3 in Anderson and Cheeger \cite{anderson-cheeger}
asserts that under the given hypotheses there is a 
harmonic coordinate chart $u : B(z,r_0) \to \real^n$, 
$E := u(B(z,r_0))$, such that \\
(1') $u(z) = 0$; \\
(2') $r_0 = r_0(n,\kappa_0,i_0,\alpha,Q)$; \\
(3') the coefficients 
$g_{ij}(u) = g(\frac{\partial}{\partial u^i}, \frac{\partial}{\partial u^j})$ 
of the Riemannian metric satisfy $g_{ij}(0) = \delta_{ij}$ and, 
with $p$ defined by $\alpha =1-n/p$,
\begin{align}
Q^{-1} & \norm{\xi}_{\real^n}^2 
\leq \sum_{ij} \, \xi^i \xi^j g_{ij}(u) 
\leq Q \norm{\xi}_{\real^n}^2 
&& (\forall \, u \in E, \, \forall \, \xi \in \real^n); \label{E:metricbound} \\
& r_0^{\alpha} \norm{\nabla g_{ij}}_{L^p(E)} \leq Q -1 
&& (\forall \, i,j). \label{E:metriclip}
\end{align} 
\par
First, we put $r_h := r_0 / \sqrt{Q}$ and show that $B = B(r_h) \subseteq E$.
Fix a unit vector $v \in \real^n$, and put $\gamma(t) = t v$.
Note $\norm{\gamma'(t)}_{g}^2 \leq Q$ by \eqref{E:metricbound}.
Let $L(\cdot)$ denote the length function on curves in $M$.
Then $t < r_0/\sqrt{Q}$ implies $d_M(\gamma(0),\gamma(t)) 
\leq L(\gamma |_{[0,t]}) \leq t \sqrt{Q} < r_0$. 
\par
Second, by Morrey's inequality,  
there is a constant $C = C(n,\alpha,r_0)$ for which
$\norm{g_{ij}}_{C^{\alpha}(\bar{B})} 
\leq C \norm{g_{ij}}_{W^{1,p}(B)}$.
Then, by \eqref{E:metricbound} and \eqref{E:metriclip},
there is a constant $C = C(n,\alpha,r_0,Q)$ such that
$[g_{ij}]_{C^{\alpha}(\bar{B})} \leq
\norm{g_{ij}}_{C^{\alpha}(\bar{B})} \leq C$ for all $i,j$.
\par
Third, note that bounds \eqref{E:metricbound} and \eqref{E:invmetricbound}
on the metric and its inverse are equivalent.
\par
Fourth, we show that $[g^{ij}]_{C^{\alpha}} 
= [g^{ij}]_{C^{\alpha}(\bar{B})}$ is bounded.
For $x,y \in B$, put $A := (g_{ij}(x))$ and $B := (g_{ij}(y))$.
We use $\norm{\cdot}_2$ to denote the induced 2-norm on matrices
in $\real^{n \times n}$, and $\norm{\cdot}_{max}$ to denote
the largest magnitude over entries of a matrix in  $\real^{n \times n}$.
Note $\norm{\cdot}_{max} \leq \norm{\cdot}_{2} \leq n \norm{\cdot}_{max}$,
$\norm{A^{-1}}_2 \leq Q$, $\norm{B^{-1}}_2 \leq Q$,
and $A^{-1}-B^{-1} = -A^{-1}(A-B)B^{-1}$. Hence
\begin{align}
\abs{g^{ij}(x)-g^{ij}(y)} &\leq \norm{A^{-1}-B^{-1}}_{max} \\
&\leq \norm{A^{-1}-B^{-1}}_{2} \\
&= \norm{A^{-1}(A-B)B^{-1}}_{2} \\
&\leq \norm{A^{-1}}_2 \norm{A-B}_2  \norm{B^{-1}}_{2} \\
&\leq n Q^2 \cdot \; \max_{kl} \, \abs{g_{kl}(x)-g_{kl}(y)}
\end{align}
It follows that
$[g^{ij}]_{C^{\alpha}} \leq n Q^2 C$ for all $i, j$.
\end{proof}

Using harmonic coordinates and the eigenvalue bound
with Theorem \ref{T:jms}, we finish the proof.
\begin{proof}[Proof of Theorem \ref{T:localed}]
Fix $Q >1$ and $\alpha <1$. Our choice of $n, \kappa_0, i_0$ then
fixes the constants $r_h, C_h$ for harmonic coordinates.
Use harmonic coordinates in Theorem \ref{T:jms}
with $C_1 = Q$, $C_2 = C_h$, and $r = r_h$.
These determine the constants $\nu = \nu(n,C_1,C_2)$ and 
$C_{\lambda} = C_{\lambda}(n,\kappa_0,D)$ in 
Theorems \ref{T:jms} and \ref{T:eigenvaluebounds}, respectively.
Let $m+1$ be the smallest integer such that 
$C_{\lambda} (m+1)^{2/n} > \nu r_h^{-2}$.
Now, for any $M \in \mathcal{M}$,
$\lambda_{m+1}(M) \geq C_{\lambda} (m+1)^{2/n} > \nu r_h^{-2}$.
It follows from Theorem \ref{T:jms} that 
$\Phi^m_z(M) : B(z, \epsilon) \to \real^m$
is an embedding with $\epsilon = \nu^{-1} r_h$.
\end{proof}

%%%%%%%%%%%%%%%%%%%%%%%%%%%%%%%%%%%%%%%%%%%%%%%%%%%%%%%%%%%%%%%%%%%%%%%%%%%%%%%
% SECTION : Proof 2
%%%%%%%%%%%%%%%%%%%%%%%%%%%%%%%%%%%%%%%%%%%%%%%%%%%%%%%%%%%%%%%%%%%%%%%%%%%%%%%

\section{Proof of Theorem \ref{T:ued}}

The proof of Theorem \ref{T:ued} builds on Theorem \ref{T:localed},
extending injectivity to the whole manifold via heat kernel estimates.
In particular, a Gaussian bound for the heat kernel will be extended to the partial sum
\begin{equation}
p^{k}(t,x,y) := \sum_{j=0}^k \, e^{-\lambda_j t} \varphi_j(x) \varphi_j(y)
\end{equation}
through a universal bound for the remainder term.

\subsection{Off-diagonal Gaussian upper bound for the heat kernel}

\begin{theorem}[Li-Yau \cite{liyau86}]
\label{T:liyau}
Let $M$ be a complete $n$-dimensional Riemannian manifold
without boundary and with $\ricci_M \geq -(n-1) \kappa_0 g$ ($\kappa_0 \geq 0$).
Put $V_x(r) = \vol(B(x,r))$. Then, for $0 < \delta < 1$, 
the heat kernel satisfies
\begin{equation*}
\begin{split}
p(t,x,y) \leq \frac{C(n,\delta)}{V_x^{1/2}(\sqrt{t})\, V_y^{1/2}(\sqrt{t})}
\exp \bigg\{\! -\frac{d^2(x,y)}{(4+\delta) t} +C(n) \kappa_0 t \, \bigg\} \,.
\end{split}
\end{equation*}
for all $t>0$ and $x,y \in M$. Moreover, $C(n,\delta) \to \infty$ as $\delta \to 0$.
\end{theorem}

\begin{theorem}[Croke \cite{croke80}]
Let $M$ be an $n$-dimensional Riemannian manifold.
Then there is a constant $C_n$ depending only on $n$ such that, for all 
$x \in M$, for all $r \leq \frac{1}{2} \inj(M)$,
\begin{equation*}
\vol(B(x,r)) \geq C_n \, r^n .
\end{equation*}
\end{theorem}

\begin{corollary}
\label{T:upperg}
There is a constant $C_U >0$ such that,
for any $M \in \mathcal{M}$, for any $x, y \in M$, for any $t \in (0,i_0/2]$,
\begin{equation*}
p(t,x,y; M) \leq \frac{C_U}{t^{n/2}} 
\exp \bigg\{\! -\frac{d^2(x,y)}{(4+\delta) t} \, \bigg\} \,,
\end{equation*} 
where $0 < \delta < 1$.
\end{corollary}
\begin{proof}
Put $\delta =1/2$.
Applying Croke's estimate to the Li-Yau heat kernel bound,
\begin{equation*}
p(t,x,y; M) \leq \frac{C(n,\delta)}{C_n \,t^{n/2}}
\exp \bigg\{\! -\frac{d^2(x,y)}{(4+\delta) t} +C(n) \kappa_0 i_0/2 \, \bigg\} \,.
\end{equation*}
\end{proof}

\subsection{Truncating the heat kernel sum}
We consider control over $M \in \mathcal{M}$ of the remainder term
\begin{equation}
R_k(t; M) := \sup_{x \in M} \, \sum_{j \geq k} 
e^{- \lambda_j t} \varphi_j^2(x) .
\end{equation}

\begin{lemma}
For all $k \in \naturals$, there is $E_k : \real^+ \to \real^+$ such that, for all 
$M \in \mathcal{M}$,
\begin{equation}
R_k(t; M) \leq E_k(t) \,t^{-n/2} \,,
\end{equation}
and $\lim_{k \to \infty} E_k(t) = 0$ for fixed $t > 0$.
\end{lemma}
\begin{proof}
From the proof of Theorem 17 in \cite{bbg94} (p.\,393), 
there exists $E_0 = E_0(n,\kappa_0,D)$ such that
\begin{equation}
R_k(t; M) \leq 
E_0 \,t^{-n/2} \int_{t \lambda_{k}}^\infty s^{n/2} e^{-s} \, ds.
\end{equation}
Now recall from Theorem \ref{T:eigenvaluebounds} above that
$C_{\lambda} k^{2/n} \leq \lambda_k$, where 
$C_{\lambda}=C_{\lambda}(n,\kappa_0,D)$. Put
\begin{equation}
E_k(t) := E_0 \int_{C_{\lambda} k^{2/n} t}^\infty s^{n/2} e^{-s} \, ds.
\end{equation}
Hence $R_k(t; M) \leq E_k(t) \,t^{-n/2}$ and
$\lim_{k \to \infty} E_k(t) = 0$ for fixed $t>0$.
\end{proof}

\subsection{Final steps}

Now take $\epsilon >0$ and $m \in \naturals$ from Theorem \ref{T:localed}. Put
\begin{equation}
\label{E:defgbound}
g(t) := 1 - \frac{C_U}{t^{n/2}} 
\exp \bigg\{ \frac{- \epsilon^2}{(4+\delta) t} \bigg\} \,.
\end{equation}
Let $M \in \mathcal{M}$, and let $p$ be its heat kernel. 
Note the bound $p(t,x,x) \geq \varphi_0^2(x) = \vol(M)^{-1} = 1$, 
which follows from the series expansion \eqref{equation:hksum} of the heat kernel.
Then, combined with Corollary \ref{T:upperg}, for $t \in (0,i_0/2]$,
\begin{equation}
g(t)  \leq \inf_{d_M(x,y) \geq \epsilon} \; p(t,x,x) -p(t,x,y) \,,
\end{equation}
and $g(t) \to 1$ as $t \to 0^+$.
Choose $T \in (0,i_0/2]$ to satisfy $g(T) \geq 4/5$;
then choose $d \geq m$ satisfying $E_{d+1}(T) \,T^{-n/2} \leq 1/5$. 
We now complete the proof.
\begin{proof}[Proof of Theorem \ref{T:ued}]
By Theorem \ref{T:localed}, since $d \geq m$, we already know that 
$\Phi^d$ is an immersion and that it distinguishes points within
distance $\epsilon$ of one another.  Suppose $d_M(x,y) \geq \epsilon$.
Then, noting 
\begin{align}
\sup_{x',y' \in M} \; \abs{ \, p(T,x',y')-p^d(T,x',y') \,} \leq R_{d+1}(T; M) \leq 1/5 ,
\end{align}
we have
\begin{align}
4/5 &\leq g(t) \\
&\leq p(T,x,x) - p(T,x,y) \\
&\leq  p^d(T,x,x) -p^d(T,x,y) + 2/5 \,,
\end{align}
hence $p^d(T,x,x) >p^d(T,x,y)$.
Finally, observe that $p^d(T,x,x) \neq p^d(T,x,y)$ implies 
$\Phi^d(x) \neq \Phi^d(y)$.
\end{proof}

\begin{remark}
Note that were we able to explicitly compute $\epsilon$ and $m$ in Theorem \ref{T:localed}, 
we could also write an explicit bound for the maximal embedding dimension $d$ as follows.  
The foregoing proof reduces to finding the smallest $d \geq m$
for which $g(t) > 2 E_{d+1}(t) \,t^{-n/2}$ is satisfied for some $t \in (0,i_0/2]$.
Moreover, to achieve a tighter bound, we could improve the lower bound $g(t)$ from 
\eqref{E:defgbound} above to
\begin{equation}
\frac{C_L}{t^{n/2}} - \frac{C_U}{t^{n/2}} 
\exp \bigg\{ \frac{- \epsilon^2}{(4+\delta) t} \bigg\} \,,
\end{equation}
where $p(t,x,x) \geq C_L/t^{n/2}$ for all $t \in (0,i_0/2]$, $C_L = C_L(n,\kappa_0)$,
follows from the on-diagonal 
Gaussian lower bound for the heat kernel (cf.\ \cite{cheegeryau81,daviesmandouvalos}).
Explicit computation of the maximal embedding dimension would then reduce to writing out the constants 
$C_L$, $C_U$, $E_0$, and $C_{\lambda}$.
One can use the formulas in \cite{grigoryan-noguchi} to compute $C_L$,
the formulas in \cite{berard86,grigoryan09} to compute $C_U$,
and the formulas in \cite{bbg94,berard86}, along with Croke's estimate
to establish the uniform diameter bound $D$, to compute $E_0$ and $C_{\lambda}$.

However, in this note, both $\epsilon$ and $m$ depend on the scaled ``harmonic radius'' $r_h$
of Lemma \ref{T:ac},
whose dependency on injectivity radius and Ricci curvature 
is established by indirect means (proof by contradiction) 
in Anderson and Cheeger \cite{anderson-cheeger},
and the author of this note has not pursued deriving a formula for $r_h$
in terms of injectivity radius and Ricci curvature.
\end{remark}

%%%%%%%%%%%%%%%%%%%%%%%%%%%%%%%%%%%%%%%%%%%%%%%%%%%%%%%%%%%%%%%%%%%%%%%%%%%%%%%
% SECTION : Proof 3
%%%%%%%%%%%%%%%%%%%%%%%%%%%%%%%%%%%%%%%%%%%%%%%%%%%%%%%%%%%%%%%%%%%%%%%%%%%%%%%

\section{Proof of Theorem \ref{T:ueds}}

The last theorem derives from the following two results.

\begin{theorem}[Cheeger \cite{cheeger70}]
\label{T:klingenbergcheeger}
Let $K$ denote the sectional curvature of a complete Riemannian manifold $M$.
If $\abs{K} \leq \kappa_0$, $\vol(M) \geq V_0$, $\diam(M) \leq D$,
then $\inj(M) \geq i_0$ for some $i_0 = i_0(n, \kappa_0, V_0, D)$.
\end{theorem}

\begin{theorem}[Topping \cite{topping08}]
\label{T:topping}
Let $M$ be a closed $n$-dimensional Riemannian manifold 
isometrically immersed in $\real^k$ with mean curvature vector $H$. 
There is a constant $C=C(n)$ such that
\begin{equation}
\diam(M) \leq C \int_M \, \abs{H}^{n-1} \,dV \,.
\end{equation}
\end{theorem}

Recall
\begin{equation*}
\begin{split}
\mathcal{S} := \{ \; (M,g) \mid \;
& \dim M = 2, \; \abs{K} \leq \kappa_0, \\
& \abs{H} \leq H_0, \; \vol(M) \leq A, \\ 
& \iota : M \hookrightarrow \real^3\ \text{is an isometric immersion} \;\} \,.
\end{split}
\end{equation*}
Note that if $(M,g) \in \mathcal{S}$ and we scale $g$ by $a>0$
so that $\vol(M,a^2 g) = 1$, then 
$1 = \vol(M,a^2 g) = a^2 \vol(M,g) \leq a^2 A$, or, $a^{-1} \leq A^{1/2}$.
Noting $K(M,a^2g) = a^{-2} K(M,g)$ and $H(M,a^2g) = a^{-1} H(M,g)$,
we have  
\begin{equation}
\begin{split}
(M,a^2 g) \in \{ \; (M,g) \mid \; 
&\dim M = 2, \; \abs{K} \leq A \kappa_0, \\
& \abs{H} \leq A^{1/2} H_0, \; \vol(M) = 1, \\
& \iota : M \hookrightarrow \real^3\ \text{is an isometric immersion} \; \} \,.
\end{split}
\end{equation}
For surfaces, note that Gaussian curvature $K$ and sectional curvature coincide,
and $K$ is related to Ricci curvature by $\ricci_M = K g$.
Applying Theorem \ref{T:topping}, then Theorem \ref{T:klingenbergcheeger},
reduces the present case to that of Theorem \ref{T:ued}.
It follows that $\mathcal{S}$ has a uniform embedding dimension.

\section*{Acknowledgements}

The human model surfaces (triangle meshes) were provided 
by the McGill 3D Shape Benchmark
(http://www.cim.mcgill.ca/\%7eshape/benchMark/).

The hippocampus surface was segmented from a baseline MR scan
from the Alzheimer’s Disease Neuroimaging Initiative (ADNI) database (adni.loni.ucla.edu). 
For up-to-date information, see www.adni-info.org.
Segmentation was done by the Florida State University Imaging Lab
using the software \emph{FreeSurfer} \cite{freesurfer1,freesurfer2}.
We thank Xiuwen Liu, Dominic Pafundi, and Prabesh Kanel
for their help with this data.

% \appendix

%% References


\bibliographystyle{elsarticle-num}
\bibliography{embeddings02}   % name your BibTeX data base

\begin{thebibliography}{10}
\expandafter\ifx\csname url\endcsname\relax
  \def\url#1{\texttt{#1}}\fi
\expandafter\ifx\csname urlprefix\endcsname\relax\def\urlprefix{URL }\fi
\expandafter\ifx\csname href\endcsname\relax
  \def\href#1#2{#2} \def\path#1{#1}\fi

\bibitem{abdallah12}
H.~Abdallah, Embedding {R}iemannian manifolds via their eigenfunctions and
  their heat kernel, Bull. Korean Math. Soc. 49~(5) (2012) 939--947.

\bibitem{berardvolume}
P.~B{\'e}rard, Volume des ensembles nodaux des fonctions propres du laplacien,
  S\'eminare Bony-Sj\"{o}strand-Meyer, \'Ecole Polytechnique, 1984-1985,
  expos\'e no.\ 14.

\bibitem{jms08}
P.~Jones, M.~Maggioni, R.~Schul, Manifold parametrizations by eigenfunctions of
  the {L}aplacian and heat kernels, Proc. Natl. Acad. Sci. USA 105~(6) (2008)
  1803--1808.

\bibitem{jms10}
P.~W. Jones, M.~Maggioni, R.~Schul, Universal local parametrizations via heat
  kernels and eigenfunctions of the {L}aplacian, Ann. Acad. Scient. Fen. 35
  (2010) 1--44.

\bibitem{bbg88}
P.~B{\'e}rard, G.~Besson, S.~Gallot, On embedding {R}iemannian manifolds in a
  {H}ilbert space using their heat kernels, {P}r\'epublication de l'Institut
  Fourier, no.\ 109 (1988).

\bibitem{bbg94}
P.~B\'erard, G.~Besson, S.~Gallot, Embedding {R}iemannian manifolds by their
  heat kernel, Geom. Funct. Anal. 4~(4) (1994) 373--398.

\bibitem{fukaya87}
K.~Fukaya, Collapsing of {R}iemannian manifolds and eigenvalues of {L}aplace
  operator, Invent. Math. 87~(3) (1987) 517--547.

\bibitem{kasuekumura94}
A.~Kasue, H.~Kumura, Spectral convergence of {R}iemannian manifolds, Tohoku
  Math. J. (2) 46~(2) (1994) 147--179.

\bibitem{kasuekumura96}
A.~Kasue, H.~Kumura, Spectral convergence of {R}iemannian manifolds {II},
  Tohoku Math. J. 48~(1) (1996) 71--120.

\bibitem{kasue97}
A.~Kasue, H.~Kumura, Y.~Ogura, Convergence of heat kernels on a compact
  manifold, Kyushu J. Math. 51~(2) (1997) 453--524.

\bibitem{kasue02}
A.~Kasue, Convergence of {R}iemannian manifolds and {L}aplace operators {I},
  in: Ann. Inst. Fourier, Vol.~52, Chartres: l'Institut, 2002, pp. 1219--1258.

\bibitem{kasue06}
A.~Kasue, Convergence of {R}iemannian manifolds and {L}aplace operators {II},
  Potent. Anal. 24~(2) (2006) 137--194.

\bibitem{belkin-eigenmaps01}
M.~Belkin, P.~Niyogi, Laplacian eigenmaps and spectral techniques for embedding
  and clustering, in: Adv. Neural Inf. Process. Syst. (NIPS), Vol.~14, MIT
  Press, 2001, pp. 585--591.

\bibitem{bai-hancock04}
X.~Bai, E.~R. Hancock, Heat kernels, manifolds and graph embedding, in:
  A.~L.~N. Fred, T.~Caelli, R.~P.~W. Duin, A.~C. Campilho, D.~de~Ridder (Eds.),
  SSPR/SPR, Vol. 3138 of Lecture Notes in Computer Science, Springer, 2004, pp.
  198--206.

\bibitem{lafon-thesis}
S.~Lafon, Diffusion maps and geometric harmonics, Ph.D. thesis, Yale University
  (2004).

\bibitem{coifman-lafon06}
R.~R. Coifman, S.~Lafon, Diffusion maps, Appl. Comput. Harmon. Anal. 21~(1)
  (2006) 5--30.

\bibitem{levy06}
B.~L\'evy, Laplace-{B}eltrami eigenfunctions: Towards an algorithm that
  understands geometry, in: Int. Conf. Shape Model. Appl. (SMI), 2006.

\bibitem{rustamov07}
R.~M. Rustamov, Laplace-{B}eltrami eigenfunctions for deformation invariant
  shape representation, in: Symp. Geom. Process. (SGP), 2007, pp. 225--233.

\bibitem{jain-retrieval07}
V.~Jain, H.~Zhang, A spectral approach to shape-based retrieval of articulated
  {3D} models, Comput. Aided Des. 39~(5) (2007) 398--407.

\bibitem{elghawalby-hancock08}
H.~ElGhawalby, E.~Hancock, Measuring graph similarity using spectral geometry,
  in: A.~Campilho, M.~Kamel (Eds.), Image Analysis and Recognition, Vol. 5112
  of Lecture Notes in Computer Science, Springer Berlin Heidelberg, 2008, pp.
  517--526.
\newblock \href {http://dx.doi.org/10.1007/978-3-540-69812-8_51}
  {\path{doi:10.1007/978-3-540-69812-8_51}}.

\bibitem{bates-icpr10}
J.~Bates, X.~Liu, W.~Mio, Scale-space spectral representation of shape, in:
  Proc. Int. Conf. Pattern Recognit. (ICPR), 2010, pp. 2648--2651.

\bibitem{memoli-spectral}
F.~M\'emoli, A spectral notion of {Gromov-Wasserstein} distance and related
  methods, Appl. Comput. Harmon. Anal. 30~(3) (2011) 363--401.

\bibitem{carcassoni-hancock00}
M.~Carcassoni, E.~R. Hancock, Spectral correspondence for deformed point-set
  matching, in: H.-H. Nagel, F.~J.~P. L{\'o}pez (Eds.), AMDO, Vol. 1899 of
  Lecture Notes in Computer Science, Springer, 2000, pp. 120--132.

\bibitem{jain-correspondence07}
V.~Jain, H.~Zhang, O.~van Kaick, Non-rigid spectral correspondence of triangle
  meshes, Int. J. Shape Model. 13~(1) (2007) 101--124.

\bibitem{mateus-etal08}
D.~Mateus, R.~Horaud, D.~Knossow, F.~Cuzzolin, E.~Boyer, Articulated shape
  matching using {L}aplacian eigenfunctions and unsupervised point
  registration, in: Conf. Comput. Vis. Pattern Recognit. (CVPR), 2008, pp.
  1--8.

\bibitem{liu-icpr08}
X.~Liu, A.~Donate, M.~Jemison, W.~Mio, Kernel functions for robust {3D} surface
  registration with spectral embeddings, in: Proc. Int. Conf. Pattern Recognit.
  (ICPR), 2008, pp. 1--4.

\bibitem{bates-isbi09}
J.~Bates, Y.~Wang, X.~Liu, W.~Mio, Registration of contours of brain structures
  through a heat-kernel representation of shape, in: Int. Symp. Biomedical
  Imaging (ISBI), 2009, pp. 943--946.

\bibitem{reuter-ijcv10}
M.~Reuter, Hierarchical shape segmentation and registration via topological
  features of {L}aplace-{B}eltrami eigenfunctions, Int. J. Comput. Vis. 89~(2)
  (2010) 287--308.

\bibitem{sharma-horaud}
A.~Sharma, R.~Horaud, Shape matching based on diffusion embedding and on mutual
  isometric consistency, in: Conf. Comput. Vis. Pattern Recognit. Workshops
  (CVPRW), 2010, pp. 29--36.

\bibitem{vonluxburg-consistency}
U.~{von Luxburg}, M.~Belkin, O.~Bousquet, Consistency of spectral clustering,
  Ann. Statistics (2008) 555--586.

\bibitem{belkin-niyogi-convergence}
M.~Belkin, P.~Niyogi, Convergence of {L}aplacian eigenmaps, Adv. Neural Inf.
  Process. Syst. (NIPS) 19.

\bibitem{ting2011analysis}
D.~Ting, L.~Huang, M.~Jordan, An analysis of the convergence of graph
  {L}aplacians, preprint ar{X}iv:1101.5435.

\bibitem{berard86}
P.~B\'{e}rard, Spectral geometry: direct and inverse problems, Lecture notes in
  math., Springer-Verlag, 1986.

\bibitem{chavel84}
I.~Chavel, Eigenvalues in {R}iemannian geometry, Pure and applied math.,
  Academic Press, 1984.

\bibitem{rosenberg97}
S.~Rosenberg, The {L}aplacian on a {R}iemannian Manifold, Cambridge University
  Press, 1997.

\bibitem{grigoryan09}
A.~Grigor'yan, Heat kernel and analysis on manifolds, Studies in advanced
  mathematics, AMS/IP, 2009.

\bibitem{zelditch-eigenfunctions}
S.~Zelditch, Local and global analysis of eigenfunctions on {R}iemannian
  manifolds, in: L.~Ji, P.~Li, R.~Schoen, L.~Simon (Eds.), Handbook of
  Geometric Analysis, Vol.~7 of Adv. Lect. Math., 2008, pp. 545--658.

\bibitem{evanspde}
L.~C. Evans, Partial differential equations, AMS, 1998.

\bibitem{sabitov-shefel}
I.~K. Sabitov, S.~Z. Shefel', \href{http://dx.doi.org/10.1007/BF00971679}{The
  connections between the order of smoothness of a surface and its metric},
  Siberian Math. J. 17~(4) (1976) 687--694.
\newblock \href {http://dx.doi.org/10.1007/BF00971679}
  {\path{doi:10.1007/BF00971679}}.
\newline\urlprefix\url{http://dx.doi.org/10.1007/BF00971679}

\bibitem{kazdan-deturck}
D.~M. De{T}urck, J.~L. Kazdan, Some regularity theorems in {R}iemannian
  geometry, Ann. Scient. \'Ec. Norm. Sup. 14~(3) (1981) 249--260.

\bibitem{anderson-cheeger}
M.~Anderson, J.~Cheeger, {$C^{\alpha}$}-compactness for manifolds with {R}icci
  curvature and injectivity radius bounded below, J. Differential Geom. 35
  (1992) 265--281.

\bibitem{liyau86}
P.~Li, S.-T. Yau, On the parabolic kernel of the {S}chr{\"o}dinger operator,
  Acta Math. 156~(1) (1986) 153--201.

\bibitem{croke80}
C.~Croke, Some isoperimetric inequalities and eigenvalue estimates, Ann. Sci.
  \'Ecole Norm. Sup. 13~(4) (1980) 419--435.

\bibitem{cheegeryau81}
J.~Cheeger, S.-T. Yau, A lower bound for the heat kernel, Comm. Pure Appl.
  Math. 34~(4) (1981) 465--480.

\bibitem{daviesmandouvalos}
E.~Davies, N.~Mandouvalos, Heat kernel bounds on hyperbolic space and
  {K}leinian groups, Proc. London Math. Soc. 3~(1) (1988) 182--208.

\bibitem{grigoryan-noguchi}
A.~Grigor'yan, M.~Noguchi, The heat kernel on hyperbolic space, Bull. London
  Math. Soc. 30~(6) (1998) 643--650.
\newblock \href {http://dx.doi.org/10.1112/S0024609398004780}
  {\path{doi:10.1112/S0024609398004780}}.

\bibitem{cheeger70}
J.~Cheeger, Finiteness theorems for {R}iemannian manifolds, Amer. J. Math.
  92~(1) (1970) 61--74.

\bibitem{topping08}
P.~Topping, Relating diameter and mean curvature for submanifolds of
  {E}uclidean space, Comment. Math. Helv. 83~(3) (2008) 539--546.

\bibitem{freesurfer1}
B.~Fischl, D.~Salat, A.~van~der Kouwe, N.~Makris, F.~Ségonne, A.~Dale,
  Sequence-independent segmentation of magnetic resonance images, NeuroImage 23
  (2004) S69--S84.

\bibitem{freesurfer2}
B.~Fischl, D.~Salat, E.~Busa, M.~Albert, M.~Dieterich, C.~Haselgrove,
  A.~van~der Kouwe, R.~Killiany, D.~Kennedy, S.~Klaveness, A.~Montillo,
  N.~Makris, B.~Rosen, A.~Dale, Whole brain segmentation. {A}utomated labeling
  of neuroanatomical structures in the human brain, Neuron 33~(3) (2002)
  341--355.

\end{thebibliography}


\end{document}